\documentclass[twoside,letterpaper]{article}
\usepackage[preprint]{aistats2027}
\usepackage{amsmath,amssymb,mathtools,amsthm}
\usepackage{graphicx,booktabs,array,microtype,placeins,xurl}

\usepackage[hidelinks]{hyperref}
\newtheorem{theorem}{Theorem}
\newtheorem{lemma}{Lemma}
\newtheorem{corollary}{Corollary}
\newtheorem{proposition}{Proposition}

\newcommand{\EE}{\mathbb E}
\newcommand{\PP}{\mathbb P}
\newcommand{\norm}[1]{\lVert#1\rVert_2}
\newcommand{\eps}{\varepsilon}
\newcommand{\ws}{w^\star}
\newcommand{\cE}{\mathcal E}
\newcommand{\cW}{\mathcal W}
\DeclareMathOperator{\tr}{tr}
\DeclareMathOperator{\diag}{diag}
\hypersetup{pdftitle={Tail Geometry and Statistical Efficiency of Adversarial Ramp Fitting},pdfauthor={Kunyu Wang, Dehan Wang, Wenjun Chen}}

\begin{document}
\runningtitle{Tail Geometry and Statistical Efficiency of Adversarial Ramp Fitting}
\runningauthor{Kunyu Wang, Dehan Wang, and Wenjun Chen}
\twocolumn[
\aistatstitle{Tail Geometry and Statistical Efficiency of Adversarial Ramp Fitting}
\aistatsauthor{Kunyu Wang$^{1}$ \And Dehan Wang$^{1}$ \And Wenjun Chen$^{2}$}
\aistatsaddress{$^{1}$The Chinese University of Hong Kong \quad $^{2}$Columbia University\\
\texttt{kunyuwang@link.cuhk.edu.hk}\quad \texttt{dhwang@link.cuhk.edu.hk}\\
\texttt{wc2741@columbia.edu}}]
\begin{abstract}
Empirical ramp fitting can assign weight to pure-noise features even when the population optimum ignores them. We quantify this gap for norm-constrained adversarial classification with Gaussian signal and noise. The variance cost relative to normalized signed mean separates into two factors: selecting observations inside the active margin window and the curvature induced by the norm constraint. Changing the tail variance leaves the active-window probability unchanged but changes the second factor. With positive attack budget and a signal-only predictor of risk below one half, we prove a uniform quadratic tail-deletion bound, including at zero tail variance. Sufficiently accurate approximate global empirical minimizers admit exact fixed-dimensional asymptotic covariances in the low-risk regime with isotropic principal covariance. For positive tail variance at most principal variance, the product exceeds one; an additional moment condition transfers it to expected excess ramp and robust classification risks. A wide-window analysis characterizes when this ordering reverses. Controlled experiments test the decomposition, and a separate contamination study examines its scope outside the Gaussian training model.
\end{abstract}

\section{Introduction}
Empirical fitting can assign weight to pure noise even when the population optimum is tail-free. We study this discrepancy for norm-constrained adversarial ramp classifiers, showing how active-window selection and sphere-induced curvature jointly determine statistical efficiency.

The comparison is most revealing when the active-window probability and optimal direction stay fixed. Changing only the tail variance then changes relative efficiency through curvature, not through the number of active observations. In our nearly low-rank benchmark the combined cost is close to thirtyfold, although absolute tail error still vanishes as tail variance tends to zero. We make three contributions:
\begin{enumerate}
\item \textbf{Uniform population geometry.} With positive attack budget and a risk witness below $1/2$, tail deletion improves risk quadratically throughout a near-optimal set, including at zero tail variance (Theorem~\ref{thm:tail}).
\item \textbf{Exact relative efficiency.} Under isotropy and fixed dimension, sufficiently accurate approximate global ERM has explicit tangent covariances. The tail variance ratio to signed mean factors into active-window selection and constraint curvature; expectation limits require an additional moment condition (Theorem~\ref{thm:efficiency}, Corollaries~\ref{cor:inflation}--\ref{cor:riskseparation}).
\item \textbf{Window-dependent reversal.} A wider ramp window can reverse tail and aggregate risk orderings. Aggregate improvement requires tail gains to outweigh principal-direction costs (Proposition~\ref{prop:wide}).
\end{enumerate}

\paragraph{Related Work.}
Complexity bounds \cite{awasthi2020,yin2019}, sample-complexity separations \cite{schmidt2018} and Gaussian guarantees \cite{dan2020} study the statistical cost of robustness. Related work treats proportional limits \cite{javanmard2022}, benign overfitting \cite{chen2023}, robust representations \cite{awasthi2021representations} and latent models \cite{javanmard2024latent}. We compare two estimators at the same attack budget to isolate loss-induced costs.

Tanner et al.~\cite{tanner2025} analyze convex-loss empirical minimizers in proportional limits; Vilucchio et al.~\cite{vilucchio2024} study regularization geometry. Ho-Nguyen and Wright~\cite{honguyen2023} connect Wasserstein robust classification to regularized ramp minimization and establish uniqueness under their assumptions. Our setting instead uses a fixed-radius ramp objective and Gaussian signal and noise.

Koo et al.~\cite{koo2008} establish Bahadur representations and asymptotic covariance for linear hinge-loss SVM. Their Gaussian example links larger coefficient variance to fewer observations inside the margins. Sabato et al.~\cite{sabato2013} give a Gaussian-mixture sample-complexity separation between margin-error minimization and class-mean generative learning (Section~9, Example~2). Our fixed-dimensional ramp analysis isolates tail curvature at unchanged optimal direction and active-window probability, with exact asymptotic variances and classification-risk constants.

We use empirical-process tools \cite{bartlett2005}, related to adversarial fast rates \cite{mustafa2022} and nonsmooth inference \cite{wang2018svm}. Our Gaussian classification-risk expansion complements general calibration studies \cite{bao2020}.

\section{Model and Population Tail Structure}
Let $Y$ be uniform on $\{-1,1\}$ and
\begin{equation}\label{eq:model}
 \begin{gathered}
 X=(Z,\tau G),\qquad Z=Y\mu+\xi,\\
 \xi\sim N(0,\Lambda),\qquad G\sim N(0,I_k),
 \end{gathered}
\end{equation}
with independent $Y,\xi,G$. Here $r,k\ge1$, $B,\gamma>0$, $\eps\ge0$, $\tau\ge0$, and $\Lambda$ is positive definite. Write $w=(a,b)$, $\cW_B=\{w:\norm w\le B\}$ and $m=\norm\mu$. The exact Euclidean-adversarial margin and ramp loss are
\begin{equation}\label{eq:margin}
 \begin{gathered}
 M_w(x,y)=yw^\top x-\eps\norm w,\\
 \phi_\gamma(t)=\min\{1,\max\{0,1-t/\gamma\}\}.
 \end{gathered}
\end{equation}
For an iid sample of size $n$, let $\ell_w=\phi_\gamma(M_w)$, $R(w)=P\ell_w$, $R_n(w)=P_n\ell_w$, $R^\star=\min_{\cW_B}R$, and $\cE(w)=R(w)-R^\star$. The principal block is an analytical coordinate system, not information supplied to the full-space learner. The nearly low-rank regime has $\tau^2\ll\lambda_{\min}(\Lambda+\mu\mu^\top)$; small individual variance and small total tail energy $\tau^2k$ are distinct.

The margin is Gaussian with mean $\nu_w=a^\top\mu-\eps\norm w$ and variance $v_w=a^\top\Lambda a+\tau^2\norm b^2$. For $v>0$,
\begin{equation}\label{eq:risk}
 \begin{gathered}
 F_\gamma(\nu,v)=\frac1\gamma\int_0^\gamma
 \Phi\left(\frac{s-\nu}{\sqrt v}\right)\,ds,\\
 R(w)=F_\gamma(\nu_w,v_w).
 \end{gathered}
\end{equation}
At $v=0$, set $F_\gamma(\nu,0)=\phi_\gamma(\nu)$. This follows by integrating the indicators $\mathbf1\{M_w\le s\}$. Write $\varphi,\Phi$ for the standard normal density and distribution function.

\begin{samepage}
\begin{lemma}\label{lem:calculus}
For $v>0$ and $\pi(\nu,v)=\Phi((\gamma-\nu)/\sqrt v)-\Phi(-\nu/\sqrt v)$,
\begin{equation}\label{eq:derivatives}
 \begin{aligned}
 \partial_\nu F_\gamma&=-\pi(\nu,v)/\gamma<0,\\
 \partial_v F_\gamma&=
 \frac{\varphi((\gamma-\nu)/\sqrt v)-\varphi(\nu/\sqrt v)}
 {2\gamma\sqrt v}.
 \end{aligned}
\end{equation}
Moreover, $F_\gamma(\nu,v)<1/2$ iff $\nu>\gamma/2$, including $v=0$, and $\partial_vF_\gamma\ge0$ iff $\nu\ge\gamma/2$.
\end{lemma}
\end{samepage}
\begin{proof}
Differentiate \eqref{eq:risk}; the variance derivative is the integral of $(2\sqrt v)^{-1}\partial_s\varphi((s-\nu)/\sqrt v)$. Symmetry gives $F_\gamma(\gamma/2,v)=1/2$. The mean derivative is negative, and $(\gamma-\nu)^2\le\nu^2$ iff $\nu\ge\gamma/2$. The deterministic case follows from the ramp definition.
\end{proof}

Assume $w_0=(a_0,0)$ satisfies $R(w_0)\le1/2-2\Delta$. Such a witness exists for some $\Delta>0$ iff $B(m-\eps)>\gamma/2$: the aligned boundary vector maximizes the mean margin.

\begin{theorem}[Uniform Tail Deletion]\label{thm:tail}
Suppose $\eps>0$, $0<\lambda_-I\preceq\Lambda\preceq\lambda_+I$, and $0\le\tau\le\tau_{\max}<\infty$. Assume $\Delta>0$ and a witness $w_0=(a_0,0)\in\cW_B$ with $R(w_0)\le1/2-2\Delta$. Define
\begin{equation}\label{eq:tailconstants}
 \begin{aligned}
 v_{\max}&=B^2\max\{\lambda_+,\tau_{\max}^2\},\\
 \kappa&=\frac{\gamma}{\sqrt{2\pi v_{\max}}}
       \exp\left(-\frac{(m-\eps)^2}{2\lambda_-}\right),\\
 c_b&=\frac{\eps\kappa}{2\gamma B}.
 \end{aligned}
\end{equation}
Every $w=(a,b)\in\cW_B$ with $\cE(w)\le\Delta$ obeys
\[
 R(a,b)-R(a,0)\ge c_b\norm b^2,\qquad
 \norm b^2\le\cE(w)/c_b.
\]
All population minimizers have $b=0$, including when $\tau=0$.
\end{theorem}
\begin{proof}
Near-optimality gives $R(w)\le1/2-\Delta$, hence $\nu_w>\gamma/2$. Put $q=\norm a>0$ and $\beta=\norm b$. Along $w_u=(a,\sqrt u\,b)$, $0\le u\le1$,
\[
 \nu_u=a^\top\mu-\eps\sqrt{q^2+u\beta^2},\quad
 v_u=a^\top\Lambda a+u\tau^2\beta^2.
\]
Thus $\gamma/2<\nu_u\le(m-\eps)q$ and $\lambda_-q^2\le v_u\le v_{\max}$. For $s\in[0,\gamma]$, $|s-\nu_u|\le\nu_u$, so $(s-\nu_u)^2/v_u\le(m-\eps)^2/\lambda_-$. The margin density on this interval is at least $\kappa/\gamma$, giving $\pi(\nu_u,v_u)\ge\kappa$. Therefore
\begin{align*}
 \frac{d}{du}R(w_u)
 &=\frac{\eps\beta^2\pi(\nu_u,v_u)}
         {2\gamma\sqrt{q^2+u\beta^2}}
   +\tau^2\beta^2\partial_vF_\gamma(\nu_u,v_u)\\
 &\ge c_b\beta^2.
\end{align*}
Integrate and use $R(a,0)\ge R^\star$. Applying the result at a minimizer gives $b=0$. At $\tau=0$, $v_u\ge\lambda_-q^2>0$, so the differentiation remains valid.
\end{proof}

The bound improves on $\norm b^2\le B^2$ only when $\cE<c_bB^2$; its constant is uniform, not sharp. At $\eps=0$, the same path still gives nonnegative deletion improvement below risk $1/2$, but it can vanish when $\tau=0$.

Tail deletion need not help outside the near-optimal region, even when a witness exists. For $r=k=1$, $\mu=2$, $\Lambda=1$, $\tau=0.15$, $\eps=0.25$ and $B=\gamma=1$, the witness condition holds, but $(-0.005,0)$ and $(-0.005,0.5)$ have risks $0.9999788271$ and $0.9989163667$. Deletion increases risk at these nonoptimal points.

\begin{proposition}[Normalization]\label{prop:norm}
Define $S_B(w)=Bw/\norm w$ for $w\ne0$, and $S_B(0)=Bv_0$ for a fixed unit vector. For $w\in\cW_B$, $\ell_{S_B(w)}\le\ell_w$ pointwise. The ball and sphere have equal empirical and population infima; normalizing a ball-valued approximate ERM preserves its tolerance.
\end{proposition}
\begin{proof}
For nonzero $w\in\cW_B$, $c=B/\norm w\ge1$, so $M_{cw}=cM_w$ and $\phi_\gamma(ct)\le\phi_\gamma(t)$. At $w=0$ the original loss is one. Sphere inclusion gives equality of infima.
\end{proof}
On the sphere the attack is the common shift $-\eps B$. It still changes which margins are active.

\section{Exact Local Efficiency}
From now on, fix $r,k$ and all model parameters, with
\begin{equation}\label{eq:isotropic}
 \Lambda=\sigma^2I_r,\quad \mu=mu_\star,\quad
 \norm{u_\star}=1,\quad B(m-\eps)>\gamma/2.
\end{equation}
Here $\sigma>0$ and $\eps,\tau\ge0$, including ordinary training. Take an orthonormal matrix $Q$ spanning $u_\star^\perp$. In analytical coordinates,
\[
 U=u_\star^\top YZ,\quad T=Q^\top YZ,\quad V=Y\tau G
\]
are independent with laws $N(m,\sigma^2)$, $N(0,\sigma^2I_{r-1})$, and $N(0,\tau^2I_k)$. When $r=1$, the $T$ block is absent. Define
\begin{equation}\label{eq:window}
 \begin{aligned}
 A&=\mathbf1\{\eps<U<\eps+\gamma/B\},\quad p=\PP(A=1),\\
 D&=\varphi((\eps+\gamma/B-m)/\sigma)
          -\varphi((\eps-m)/\sigma).
 \end{aligned}
\end{equation}
The witness implies $0<p<1$ and $D>0$.

\begin{lemma}[Population Curvature]\label{lem:curvature}
The unique sphere minimizer is $\ws=(Bu_\star,0)$. In the chart
\[
 w(t,b)=\left(\sqrt{B^2-\norm t^2-\norm b^2}\,u_\star+Qt,b\right)
\]
with $\theta=(t,b)$ near zero,
\begin{equation}\label{eq:quadratic}
 \begin{gathered}
 \cE(w(\theta))=\tfrac12\theta^\top H\theta+O(\norm\theta^4),
 \\ H=\diag(h_zI_{r-1},h_bI_k),
 \end{gathered}
\end{equation}
where the positive eigenvalues are
\begin{equation}\label{eq:hessian}
 h_z=\frac{mp}{\gamma B},\qquad
 h_b=\frac{mp+(\tau^2-\sigma^2)D/\sigma}{\gamma B}.
\end{equation}
\end{lemma}
\begin{proof}
A competitor with risk at least $1/2$ cannot be optimal. Below $1/2$, deletion cannot increase risk by Lemma~\ref{lem:calculus}, also for $\eps=0$. Alignment at fixed principal norm $q>0$ strictly increases the mean while preserving variance. Along the aligned direction $g(q)=R(qu_\star,0)$,
\[
 -g'(q)=\frac1{\gamma\sigma}\int_0^{\gamma/q}
 z\varphi((z-(m-\eps))/\sigma)\,dz>0.
\]
If $b\ne0$ on the sphere, $q<B$, so alignment followed by scaling strictly improves risk. This proves uniqueness even if deletion itself is an equality.

In the chart the exact mean and variance are
\[
 \nu=m\sqrt{B^2-\norm\theta^2}-\eps B,\qquad
 v=\sigma^2B^2+(\tau^2-\sigma^2)\norm b^2.
\]
At $\theta=0$, $F_\nu=-p/\gamma$ and $F_v=D/(2\gamma\sigma B)$. Expansion gives \eqref{eq:quadratic}--\eqref{eq:hessian}. Finally,
\[
 \gamma B h_b=mp-\sigma D+\tau^2D/\sigma
             =\EE[UA]+\tau^2D/\sigma>0,
\]
since $U>\eps\ge0$ on a positive-probability interval.
\end{proof}

The sample score at zero exists almost surely, because $B(U-\eps)$ has no mass at either kink:
\begin{equation}\label{eq:score}
 \begin{gathered}
 \psi=-\gamma^{-1}A(T,V),\\
 \EE\psi\psi^\top=\frac p{\gamma^2}
          \diag(\sigma^2I_{r-1},\tau^2I_k).
 \end{gathered}
\end{equation}
Tail displacement shortens the signal component while replacing principal variance by tail variance. Their effects combine in $h_b$, which differs from the uniform deletion constant $c_b$.

Let a measurable sphere estimator satisfy
\begin{equation}\label{eq:erm}
 R_n(\widehat w_n)\le\inf_{\norm w=B}R_n(w)+\rho_n,
 \qquad \rho_n\ge0.
\end{equation}
The same condition follows by normalizing a ball estimator with that tolerance.

\begin{theorem}[Asymptotic Linear Representation]\label{thm:efficiency}
Under \eqref{eq:isotropic} and the approximate \emph{global} ERM condition \eqref{eq:erm}, if $n\rho_n=o_p(1)$, the coordinates $\widehat t_n=Q^\top\widehat a_n,\widehat b_n$ satisfy
\begin{equation}\label{eq:linear}
 \begin{aligned}
 \widehat t_n&=\frac1{n\gamma h_z}\sum_iA_iT_i+o_p(n^{-1/2}),\\
 \widehat b_n&=\frac1{n\gamma h_b}\sum_iA_iV_i+o_p(n^{-1/2}).
 \end{aligned}
\end{equation}
Consequently $\sqrt n(\widehat t_n,\widehat b_n)$ converges to a centered Gaussian with covariance $\diag(v_zI_{r-1},v_bI_k)$, where
\begin{equation}\label{eq:variances}
 v_z=\frac{\sigma^2p}{\gamma^2h_z^2},\qquad
 v_b=\frac{\tau^2p}{\gamma^2h_b^2}.
\end{equation}
If, additionally, for some $\delta>0$,
\begin{equation}\label{eq:moment}
 \EE(n\rho_n)^{1+\delta}\longrightarrow0,
\end{equation}
then
\begin{equation}\label{eq:expected}
 \begin{aligned}
 n\EE\norm{\widehat b_n}^2&\to kv_b,\\
 n\EE\norm{\widehat t_n}^2&\to(r-1)v_z,\\
 n\EE\cE(\widehat w_n)&\to
 \tfrac12\{(r-1)h_zv_z+kh_bv_b\}.
 \end{aligned}
\end{equation}
Exact ERM and deterministic $\rho_n=o(1/n)$ satisfy the moment condition.
\end{theorem}

For $\tau>0$, the theorem also gives $n\norm{\widehat b_n}^2/v_b\Rightarrow\chi_k^2$. Neither this law nor its expectation approximation is a finite-sample confidence interval.

\section{Proof of Theorem \ref{thm:efficiency}}
We establish a root-$n$ rate, control the nonsmooth remainder uniformly, and then compare quadratic minimizers.

\paragraph{Rate and Moment Control.}
In the chart, $\norm{w(\theta)-\ws}^2=\norm\theta^2+O(\norm\theta^4)$. Lemma~\ref{lem:curvature} therefore gives positive local quadratic growth. Continuity and uniqueness on the remaining compact sphere give a constant $c_0>0$ such that
\begin{equation}\label{eq:globalgrowth}
 \cE(w)\ge c_0\norm{w-\ws}^2\quad(\norm w=B).
\end{equation}
This constant is for fixed model parameters; no uniformity in growing dimension is asserted. Let $h_w=\ell_w-\ell_{\ws}$ and $\Gamma=\EE[XX^\top]$. Since both norms equal $B$,
\[
 |h_w|\le\gamma^{-1}|Y(w-\ws)^\top X|.
\]
Contraction and \eqref{eq:globalgrowth} give
\[
 \mathfrak R_n\{h_w:\cE(w)\le s\}\le K\sqrt{s/n},
 \qquad \EE h_w^2\le V\cE(w),
\]
where $K^2=\tr\Gamma/(\gamma^2c_0)$ and $V=\|\Gamma\|_{\mathrm{op}}/(\gamma^2c_0)$. Here $\mathfrak R_n$ is the expected signed supremum without an outer absolute value. For independent Rademacher signs $\eta_i$, the local difference set lies in the ball of radius $\sqrt{s/c_0}$, whose linear supremum is $\sqrt{s/c_0}\|\sum_i\eta_iY_iX_i\|/n$. Jensen gives the stated bound.

Theorem 2.1 of \cite{bartlett2005}, with tuning parameter one and range $[-1,1]$, gives, except on an event of probability $e^{-x}$,
\begin{equation}\label{eq:concentration}
 \sup_{\cE(w)\le s}(P-P_n)h_w
 \le4K\sqrt{s/n}+\sqrt{2Vsx/n}+\frac{8x}{3n}.
\end{equation}
For $s>1$, cap the class at risk one; the bound remains valid. Put $\alpha=\min\{(512V)^{-1},3/128\}$. On the shell $s_j/2<\cE(w)\le s_j$, $s_j=2^{j+1}u/n$, use $x_j=\alpha ns_j$. If $u\ge2048K^2$, the right side of \eqref{eq:concentration} is at most $3s_j/16$. On $n\rho_n\le u/4$, an output in this shell would instead have
\[
 (P-P_n)h_{\widehat w_n}\ge\cE(\widehat w_n)-\rho_n>3s_j/8.
\]
A union over $j\ge0$ therefore gives, for $u\ge u_0:=\max\{2048K^2,\log(2)/(2\alpha),1\}$,
\begin{equation}\label{eq:risktail}
 \PP(n\cE(\widehat w_n)>u)
 \le2e^{-2\alpha u}+\PP(n\rho_n>u/4).
\end{equation}
Thus $\widehat w_n-\ws=O_p(n^{-1/2})$. Under \eqref{eq:moment}, integration gives a bounded $(1+\delta)$-moment of $n\cE(\widehat w_n)$, and hence uniform integrability of $n\norm{\widehat w_n-\ws}^2$. This concentration argument is only used for the rate and moments; its constants need not be numerically sharp.

\paragraph{Kink Remainder.}
Write $r_\theta=\ell_{w(\theta)}-\ell_{\ws}-\psi^\top\theta$ and $f_{n,h}=\sqrt n\,r_{h/\sqrt n}$. For fixed $h$, almost-sure differentiability at zero gives $f_{n,h}\to0$. On any fixed compact $h$-cube and all sufficiently large $n$, the ramp is Lipschitz and the chart derivative is bounded, so
\begin{equation}\label{eq:envelope}
 |f_{n,h}|\le C_M\norm X,\qquad
 |f_{n,h}-f_{n,g}|\le L(X,Y)\norm{h-g},
\end{equation}
where $L$ has every finite moment. The constants include $\gamma^{-1}$. These bounds hold across kink crossings.

For completeness, let $\mathbb G_n=\sqrt n(P_n-P)$ and choose an even integer $q>d:=r-1+k$. Expansion of a centered iid sum gives
\[
 \EE|\mathbb G_n(f_{n,h}-f_{n,g})|^q
 \le C_q\norm{h-g}^q.
\]
Terms containing a sample index once vanish; the others use at most $q/2$ distinct indices and are bounded using $\EE L^q$. A dyadic grid at level $j$ has $O(2^{jd})$ parent-child increments of length $O(2^{-j})$. Their expected maximum is at most $C2^{-j(1-d/q)}$, with a summable tail uniformly in $n$. On each fixed finite grid, dominated convergence in \eqref{eq:envelope} gives $\EE f_{n,h}^2\to0$, so $\mathbb G_nf_{n,h}\to0$ in $L^2$. The grids and continuity yield
\[
 \sup_{\norm h\le M}|n(P_n-P)r_{h/\sqrt n}|=o_p(1).
\]
Combining with \eqref{eq:quadratic}, uniformly on compact sets,
\begin{equation}\label{eq:lan}
 \begin{aligned}
 nP_n(\ell_{w(h/\sqrt n)}-\ell_{\ws})
 &=h^\top Z_n+\tfrac12h^\top Hh+o_p(1),\\
 Z_n&=n^{-1/2}\sum_i\psi_i.
 \end{aligned}
\end{equation}

\paragraph{Argmin and Expectations.}
The rate puts $\widehat w_n$ in the chart with probability tending to one and gives $\sqrt n\widehat\theta_n=O_p(1)$. The quadratic minimizer $h_n^0=-H^{-1}Z_n$ is also $O_p(1)$ by the iid CLT. Comparing \eqref{eq:erm} with $w(h_n^0/\sqrt n)$ in \eqref{eq:lan} on a fixed compact set gives
\[
 \tfrac12\norm{H^{1/2}(\sqrt n\widehat\theta_n-h_n^0)}^2
 \le n\rho_n+o_p(1).
\]
Positive definiteness proves \eqref{eq:linear}, and the score covariance \eqref{eq:score} proves \eqref{eq:variances}.

To see the moment step explicitly, put $W_n=n\cE(\widehat w_n)$ and integrate \eqref{eq:risktail} against $(1+\delta)u^\delta du$. This gives
\[
 \sup_n\EE W_n^{1+\delta}<\infty
\]
under \eqref{eq:moment}; the exponential term is integrable and the tolerance term contributes at most $4^{1+\delta}\EE(n\rho_n)^{1+\delta}$. Finitely many initial $n$ are harmless because $W_n\le n$. By \eqref{eq:globalgrowth}, $n\norm{\widehat w_n-\ws}^2\le W_n/c_0$, so its uniform integrability allows convergence of the second moments in \eqref{eq:variances}.

Locally, the Taylor remainder multiplied by $n$ is bounded by $C n\norm{\widehat w_n-\ws}^2\norm{\widehat w_n-\ws}^2$. The second factor is bounded and vanishes in probability; the first is uniformly integrable. Off a fixed chart neighborhood, $\cE\ge e_0>0$ and
\[
 n\PP(\text{off chart})\le e_0^{-1}\EE[W_n\mathbf1\{W_n\ge ne_0\}]\to0.
\]
Bounded risk and parameter distance then control all off-chart contributions. This proves \eqref{eq:expected} without inferring expectation convergence from weak convergence alone.

The moment assumption cannot simply be omitted. Replace an exact ERM by a fixed pure-tail vector with independent probability $1/n$. Then \eqref{eq:erm} holds with $\rho_n$ equal to the replacement indicator, so $n\rho_n=o_p(1)$ and the limiting distribution is unchanged. Yet the limit of $n\EE\norm{\widehat b_n}^2$ increases by $B^2$.

\section{Efficiency Separation and Classification}
Define the full-input signed mean and its normalization by
\[
 \bar s_n=n^{-1}\sum_iY_iX_i,\qquad
 \widehat w_{\mathrm{mean}}=S_B(\bar s_n).
\]
The delta method at $(mu_\star,0)$ gives
\begin{equation}\label{eq:mean}
 \begin{aligned}
 \widehat t_{\mathrm{mean}}&=\frac B{mn}\sum_iT_i+o_p(n^{-1/2}),\\
 \widehat b_{\mathrm{mean}}&=\frac B{mn}\sum_iV_i+o_p(n^{-1/2}).
 \end{aligned}
\end{equation}
Its tangent variances are $v_z^M=B^2\sigma^2/m^2$ and $v_b^M=B^2\tau^2/m^2$. Indeed, the signed sample mean is exactly $(\mu+\zeta_z,\zeta_b)$, with independent $\zeta_z\sim N(0,\sigma^2I_r/n)$ and $\zeta_b\sim N(0,\tau^2I_k/n)$. Thus, almost surely,
\[
 \norm{\widehat b_{\mathrm{mean}}}^2
 =\frac{B^2\norm{\zeta_b}^2}{\norm{\mu+\zeta_z}^2+\norm{\zeta_b}^2}.
\]
Within distance $m/2$ of the nonzero population mean, normalization is Lipschitz and scaled errors have bounded Gaussian moments. Outside that event, probability decays exponentially and normalized distance is at most $2B$. This proves uniform integrability and the corresponding expectation limits. The principal subspace is not supplied to this estimator.

\begin{corollary}[Variance Inflation]\label{cor:inflation}
Under Theorem~\ref{thm:efficiency}, if $0<\tau\le\sigma$, the tail asymptotic variance ratio is
\begin{equation}\label{eq:inflation}
 I_b=\frac{v_b}{v_b^M}
 =\frac{m^2p}{\gamma^2B^2h_b^2}
 =\frac1p\left(\frac{h_z}{h_b}\right)^2
 \ge\frac1p>1.
\end{equation}
Each principal tangent direction has ratio $I_z=1/p$.
\end{corollary}
\begin{proof}
Since $D>0$ and $\tau^2-\sigma^2\le0$, $0<h_b\le h_z=mp/(\gamma B)$. Substitute in \eqref{eq:variances} and \eqref{eq:mean}.
\end{proof}
At $\tau=0$ both tail variances vanish; their ratio is undefined. These are relative efficiencies, not an information bound or finite-sample ordering. In the fixed-block model, an estimator supplied with the signal subspace can set its tail identically to zero. Allowing an unknown tail mean changes the model and generally separates the ramp-optimal direction from normalized mean.

\paragraph{Paired Fluctuations.}
The estimators share their observations. Applying the joint CLT to $A_iV_i$ and $V_i$ gives
\begin{equation}\label{eq:pairedcov}
 \sqrt n\begin{pmatrix}\widehat b_n\\\widehat b_{\mathrm{mean}}\end{pmatrix}
 \Rightarrow N\left(0,
 \begin{pmatrix}v_b&c_b^{EM}\\c_b^{EM}&v_b^M\end{pmatrix}\otimes I_k\right),
\end{equation}
where $c_b^{EM}=B\tau^2p/(m\gamma h_b)$. For $\tau>0$, each corresponding pair of tail coordinates in the limiting Gaussian law has correlation $\sqrt p$, independent of curvature. Under \eqref{eq:moment}, the estimators' correlations also converge to this value: their squared scaled errors are uniformly integrable, and $2|xy|\le x^2+y^2$ controls mixed moments.

Since $h_b=\EE[UA]/(\gamma B)+\tau^2D/(\gamma B\sigma)$, inflation decreases with $\tau^2$ at fixed other parameters while $v_b=O(\tau^2)$ near zero. A large relative loss can coexist with small absolute tail error; this is not a joint $n,k,\tau$ limit.

Let $C(w)=\PP(M_w\le0)$ and $c=(m-\eps)/\sigma$. For $q=\norm a$ and positive mean margin,
\[
 \frac{a^\top\mu-\eps B}{\sqrt{\sigma^2q^2+\tau^2\norm b^2}}
 \le\frac{mq-\eps B}{\sigma q}\le c.
\]
Nonpositive mean gives error at least $1/2$, including deterministic margins. Hence $\ws$ also minimizes $C$. Directly expanding $\Phi(-\nu/\sqrt v)$ in the chart gives
\begin{equation}\label{eq:classification}
 \begin{aligned}
 C(w(\theta))-C(\ws)
 &=\tfrac12\theta^\top H_{01}\theta+O(\norm\theta^4),\\
 h_{01,z}&=\frac{m\varphi(c)}{\sigma B^2},\\
 h_{01,b}&=\frac{\varphi(c)[\eps\sigma^2+(m-\eps)\tau^2]}
                 {B^2\sigma^3}.
 \end{aligned}
\end{equation}
where $H_{01}=\diag(h_{01,z}I_{r-1},h_{01,b}I_k)$. Under \eqref{eq:moment}, the expected classification excess has leading coefficient $\{(r-1)h_{01,z}v_z+kh_{01,b}v_b\}/2$. This is obtained from the Gaussian risk, not from pointwise domination of classification loss by ramp loss.

\begin{corollary}[Expected Risk Separation]\label{cor:riskseparation}
Assume Theorem~\ref{thm:efficiency}, \eqref{eq:moment} and $0<\tau\le\sigma$. For either $L=R$ or $L=C$, define
\[
 J_L=\lim_{n\to\infty}
 \frac{\EE[L(\widehat w_n)-L(\ws)]}
      {\EE[L(\widehat w_{\mathrm{mean}})-L(\ws)]}.
\]
Then $J_L\ge1/p>1$. When $r=1$, $J_R=J_C=I_b$.
\end{corollary}
\begin{proof}
The expectation expansion for $C$ follows from \eqref{eq:classification} by the moment argument of Section 4; its boundedness controls the off-chart event. The signed-mean moments were established above. For $L$ with tangent curvatures $g_z,g_b$, the ratio is
\begin{equation}\label{eq:riskratio}
 J_L=\frac{(r-1)g_zv_z^M I_z+kg_bv_b^M I_b}
           {(r-1)g_zv_z^M+kg_bv_b^M}.
\end{equation}
The curvatures are positive for both risks when $\tau>0$; hence the denominator is positive and this is a weighted average of $I_z,I_b\ge1/p$. For $r=1$ only the tail term remains.
\end{proof}

\begin{proposition}[Wide-Window Efficiency]\label{prop:wide}
In \eqref{eq:model}, fix $r,k$ and all parameters with $\Lambda=\sigma^2I_r$, $\sigma>0$, $m>\eps\ge0$, $0<\tau\le\sigma$ and $\gamma\ge2B(m-\eps)$. The unique sphere minimizer is $\ws=(B\mu/m,0)$, with the curvatures in \eqref{eq:hessian} satisfying $0<h_z\le h_b$.

For measurable sphere-valued estimators satisfying \eqref{eq:erm}, $n\rho_n=o_p(1)$ gives \eqref{eq:linear}, \eqref{eq:variances} and \eqref{eq:pairedcov}. Under \eqref{eq:moment}, \eqref{eq:expected} and the expected classification-risk limit from \eqref{eq:classification} also hold. Signed mean retains \eqref{eq:mean}. Let $I_z=1/p$ and use the \emph{equalities} in \eqref{eq:inflation} to define $I_b$. For $D<0$,
\begin{equation}\label{eq:tailreversal}
 I_b<1\quad\Longleftrightarrow\quad
 \tau^2<\sigma^2-\frac{m\sigma(\sqrt p-p)}{|D|}.
\end{equation}
Under \eqref{eq:moment}, the expected excess-risk ratio $J_L$ for $L=R,C$ retains \eqref{eq:riskratio}, with $(g_z,g_b)=(h_z,h_b)$ or $(h_{01,z},h_{01,b})$, respectively. In particular,
\begin{equation}\label{eq:totalreversal}
 J_L<1\ \Longleftrightarrow\ kg_bv_b^M(1-I_b)
 >(r-1)g_zv_z^M(I_z-1).
\end{equation}
\end{proposition}
\begin{proof}
At fixed $q=\norm a$, alignment with $\mu$ maximizes the mean at unchanged variance. On the sphere, alignment gives $\nu(q)=mq-\eps B$ and $v(q)=\tau^2B^2+(\sigma^2-\tau^2)q^2$. Since $\nu(q)\le\gamma/2$, Lemma~\ref{lem:calculus} yields
\[
 \frac{dR(q)}{dq}=mF_\nu+2q(\sigma^2-\tau^2)F_v<0.
\]
Hence the unique optimum has $q=B$. The window gives $D\le0$ in \eqref{eq:window}, so \eqref{eq:hessian} yields $h_b\ge h_z>0$. Compactness and local quadratic growth give \eqref{eq:globalgrowth}. Section 4 applies with the same score and Hessian; \eqref{eq:mean} and \eqref{eq:classification} then give the stated risk limits and \eqref{eq:riskratio}. Substituting $\gamma Bh_b=mp+(\sigma^2-\tau^2)|D|/\sigma$ into the variance ratio proves \eqref{eq:tailreversal}; subtracting one in \eqref{eq:riskratio} proves \eqref{eq:totalreversal}.
\end{proof}

A nonpositive right side in \eqref{eq:tailreversal} rules out tail reversal; $D=0$ leaves $I_b=I_z=1/p$. At the scalar baseline with $\gamma=5$, $I_b=0.95833$ and $I_z=1.04236$. Keeping these parameters and taking $r=4,k=32$ gives $J_R=1.02554$ and $J_C=1.03954$. Thus better tail estimation need not improve total excess risk: the tail gain must outweigh the principal-tangent cost in \eqref{eq:totalreversal}.

\begin{figure*}[!t]
\includegraphics[width=\textwidth]{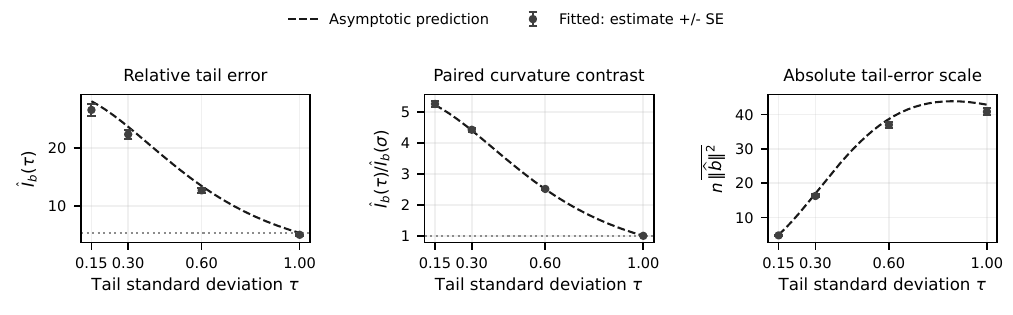}
\caption{Fixed-$p$ sweep, $r=1,k=32,n=8192$: 100 paired repetitions across $\tau$. Left: ratio of mean squared tail norms to signed mean. Center: paired ratio relative to $\tau=\sigma$. Right: absolute tail-error scale. Bars show dataset-level Monte Carlo SEs, including covariance across $\tau$ for the center contrast}
\label{fig:fixedp}
\end{figure*}

\section{Experiments}
The scalar baseline is $r=1$, $k=32$, $m=2$, $\sigma=B=\gamma=1$, $\tau=0.15$, with evaluation attack $0.25$. The main grid has 22 settings and 20 independent datasets per setting. Methods share each dataset: robust ramp, principal-space oracle ramp, ordinary ramp, normalized robust logistic and normalized signed mean. Separate 100-repetition studies cover tail variance, window width, and $r=4$ at $n=8192,32768$. We use $\mu=m\mathbf1/\sqrt r$.

\paragraph{Fitting and Reproduction.}
SLSQP uses 12 starts, at most 1000 iterations, supplied gradients and $\texttt{ftol}=10^{-10}$. Starts comprise normalized signed mean and 11 Gaussian directions with radii uniform on $[B/4,B]$. Each terminal vector is projected into the ball; the smallest training objective is selected, then normalized. Robust methods train at attack $0.25$ and ordinary ramp at zero, except in the attack sweep. This local solver does not certify \eqref{eq:erm}.

Logistic minimizes $n^{-1}\sum_i\log(1+\exp[-M^{\mathrm{train}}_{w,i}/\gamma])$. Its normalized output is a direction-estimation baseline, outside Proposition~\ref{prop:norm}'s ramp guarantee. Positive scaling preserves classification events but changes surrogate values; normalization increases logistic loss in 29 of the 440 main-grid fits.

Configurations, seed maps, all start records and fitted vectors are archived. The original scalar and $r=4$ sweeps have overlapping seed ranges and are analyzed separately; later validations use disjoint seeds. Recomputed sample losses match their saved training-objective values within $10^{-10}$. Independent Gaussian-risk quadrature matches the saved population risks within $10^{-10}$. Runs use NumPy/SciPy on CPUs, four workers and one BLAS thread each.

\subsection{Curvature at a Fixed Population Window}
At the scalar baseline, \eqref{eq:window}--\eqref{eq:inflation} give $p=0.1865682$, $h_b=0.1631106$, and $I_b=28.0500$. The selection and curvature factors are 5.3600 and 5.2332; the predicted tail and excess-ramp coefficients are 5.0490 and 0.4118.

We fix $n=8192$ and vary only $\tau\in\{0.15,0.30,0.60,1.00\}$. The same 100 draws of $(Y,\xi,G)$ keep $A_i$ in \eqref{eq:window}, evaluated at $\ws$, identical across $\tau$. The fitted active sets are different: their mean disagreement with $A_i$ is $0.112\%,0.401\%,1.228\%,2.164\%$, respectively.

Figure~\ref{fig:fixedp} estimates $I_b$ by ratios of mean squared tail norms. Observed values are $26.55,22.34,12.69,5.04$ (SEs $0.99,0.82,0.43,0.18$), against predictions $28.05,23.66,13.44,5.36$. These deviations are correlated because the datasets are shared. Using $\tau=\sigma$ as a paired reference removes the selection factor: the first three relative ratios are $5.269,4.433,2.518$ (SEs $0.088,0.066,0.025$), against curvature predictions $5.233,4.413,2.508$. SEs propagate the joint dataset-level covariance of all four means in each ratio of ratios.

\begin{figure*}[!t]
\includegraphics[width=\textwidth]{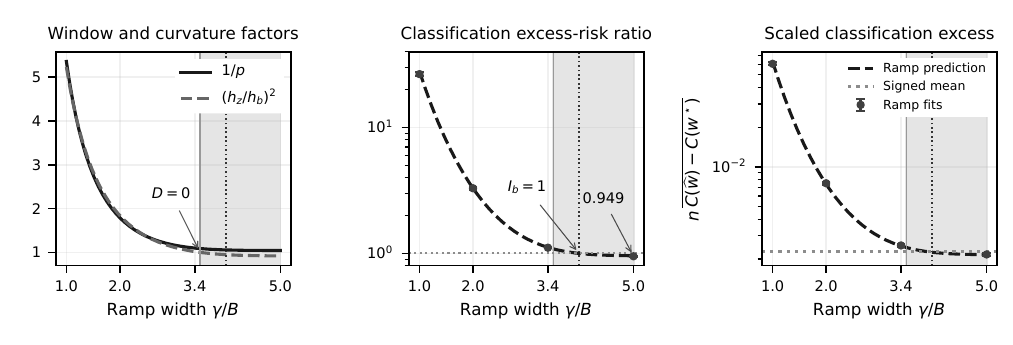}
\caption{Window-width sweep, 100 paired datasets, $r=1,k=32,n=8192$, evaluation attack $0.25$. Bars show Monte Carlo SEs. The gray boundary at 3.5 marks $D=0$; the dotted boundary at 3.9831 marks $I_b=1$. Shading denotes Proposition~\ref{prop:wide}'s regime. The annotation 0.949 is the observed classification excess-risk ratio}
\label{fig:windowwidth}
\end{figure*}

\begin{table*}[!t]
\begin{minipage}[t]{\columnwidth}
\vspace{0pt}
\caption{$r=4$, 100 independent datasets per $n$: predictions and observed means (Monte Carlo SEs). $\Delta C=C(\widehat w)-C(\ws)$; Corr.\ uses dataset-cluster SEs}
\label{tab:r4blocks}
\centering
\begin{tabular}{@{}lrrr@{}}
\toprule
Quantity&Pred.&$n=8192$&$n=32768$\\
\midrule
$n\norm{\widehat t}^2$&4.020&3.659 (0.265)&3.914 (0.288)\\
$n\norm{\widehat b}^2$&5.049&4.981 (0.127)&5.158 (0.126)\\
$n\cE$&1.162&1.090 (0.051)&1.151 (0.055)\\
$n\Delta C$&0.410&0.378 (0.023)&0.402 (0.025)\\
Corr.&0.432&0.445 (0.014)&0.429 (0.014)\\
\bottomrule
\end{tabular}
\end{minipage}\hfill
\begin{minipage}[t]{\columnwidth}
\vspace{0pt}
\caption{Contaminated training, clean evaluation: ramp/signed-mean excess-risk ratio (paired SE), $r=1,k=32,n=8192$, 100 datasets; shifts use $q=0.05$}
\label{tab:contamination}
\centering
\begin{tabular}{@{}lr@{}}
\toprule
Training setting&Risk ratio (SE)\\
\midrule
Clean&28.76 (0.99)\\
Label flips, $q=0.10$&19.93 (0.68)\\
Symmetric shift, $d=6$&0.372 (0.056)\\
Directed shift, $d=1$&0.789 (0.013)\\
Directed shift, $d=6$&0.00528 (0.00009)\\
Directed shift, $d=18$&0.000208 (0.000005)\\
\bottomrule
\end{tabular}
\end{minipage}
\end{table*}

Reducing $\tau$ lowers the absolute noise scale but increases the curvature penalty relative to signed mean. Since the population window is unchanged, this deterioration cannot be attributed to a smaller active fraction. The right panel shows why a larger efficiency ratio need not mean a larger tail error.

\subsection{Training Window and Classification Risk}
Holding the data and evaluation attack fixed, we fit $\gamma/B=1,2,3.4,5$ (Fig.~\ref{fig:windowwidth}). The first three satisfy \eqref{eq:isotropic}; the last uses Proposition~\ref{prop:wide}. Classification excess-risk ratios to signed mean are $26.565,3.294,1.111,0.949$ (SEs $0.990,0.078,0.012,0.007$), against $28.050,3.294,1.117,0.958$. Thus 28.05 is a fixed-window benchmark, and the common classification criterion shows how the gap changes with training width.

The boundaries distinguish curvature from net efficiency. At $\gamma/B=3.5$, the curvature factor equals one, but active-window selection still incurs a cost. The curvature gain offsets this cost at approximately 3.9831, also the asymptotic classification-risk transition for $r=1$. With additional principal directions, a tail improvement alone is insufficient; the weighted comparison in \eqref{eq:totalreversal} applies.

For training attacks $0,0.25,0.5$ at width one, predicted tail inflations are $70.10,28.05,13.67$. The gap therefore also exists without attacks and shrinks with attack budget in these settings; it is a loss-and-constraint effect rather than a general robustness penalty.

\subsection{Multidimensional and Solver Diagnostics}
The $r=4$ studies test both tangent blocks. At $n=8192$, the 20- and 100-repetition batches have excess-ramp means $0.00017262$ (SE $0.00001741$) and $0.00013306$ (SE $0.00000617$), a difference of 2.14 combined Monte Carlo SEs.

For each fit we compute $b_{\mathrm{lin}}=(n\gamma h_b)^{-1}\sum_i A_iV_i$ after optimization. In the independent $r=4$ batches, raising $n$ from 8192 to 32768 improves mean tail cosine from 0.9825 to 0.9915 and reduces relative residual $\norm{\widehat b-b_{\mathrm{lin}}}/\norm{\widehat b}$ from 0.1847 to 0.1283. Table~\ref{tab:r4blocks} reports both blocks, ramp and classification risk, and the paired tail correlation in \eqref{eq:pairedcov}. The principal-block deviations in Table~\ref{tab:r4blocks} are 1.36 and 0.37 SEs; they do not by themselves establish growing-dimension failure.

Across the original 40,800 starts, 143 exit unsuccessfully; all 3220 selected fits report success. Selected endpoints exceed the norm bound by at most $5.0\times10^{-11}$ before projection. Increasing 12 starts to 48 on 20 datasets per configuration $(r,n)=(1,512),(1,8192),(4,8192)$ gives objective improvements with median $6.43\times10^{-13}$, 95th percentile $7.51\times10^{-12}$ and maximum $1.68\times10^{-11}$.

For $r=k=1$, kink crossings partition the circle into sinusoidal pieces, whose endpoints and stationary minima exhaust global candidates. Double-precision enumeration gives objective gaps below $10^{-9}$ on 20 datasets at each of $n=128,512$. The multi-start checks assess stability; the enumeration check is restricted to two dimensions.

\subsection{Contaminated Training, Clean Evaluation}
We retain the scalar baseline and solver, using 100 fresh paired datasets per setting. Independently for each observation, $I\sim\mathrm{Bernoulli}(q)$ selects a tail shift $IdY\mathbf1_k$ (directed) or $IdS\mathbf1_k$ (symmetric), with independent uniform sign $S$. Shifts are added to $X_{\rm tail}$, after multiplication by $\tau$; $d=6$ is 40 clean tail standard deviations per coordinate. Label flips form a separate control. All risks use the clean distribution and attack $0.25$.

Table~\ref{tab:contamination} shows reversal for both shift types, not just mean-direction bias. Under directed shifts, the law of large numbers gives the signed-mean limit $S_B(\mu,qd\mathbf1_k)$. At $q=0.05,d=6$ its clean excess risk is 0.008524. Observed means at $n=2000,8192,32768$ are 0.008953, 0.008411, 0.008497; ramp gives $6.73,4.44,3.88$ times $10^{-5}$. Each uses 100 datasets, independent across sample sizes. These comparisons do not establish ramp consistency. Bounded loss also does not imply bounded influence: the tail linear term $AV/(\gamma h_b)$ is unbounded for active margins. The Gaussian efficiency formulas are not asserted for contaminated training.

\section{Discussion}
Tail deletion identifies unnecessary directions near the population optimum; the covariance formula measures their finite-sample estimation cost. Tail covariance and window placement determine this cost; aggregate gains also depend on principal and tail dimensions through \eqref{eq:totalreversal}. Experiments compare local fits with conditional asymptotic predictions, without certifying global near-optimality. Anisotropic principal covariance and growing dimension remain open.

The contamination controls distinguish two effects. Symmetric shifts preserve the signed-mean direction but increase sampling variability; directed shifts change its probability limit. Both can reverse finite-sample risk orderings, so mean-direction preservation alone does not preserve efficiency. Their geometry also depends on dimension: a shift of $d$ in every tail coordinate has norm $d\sqrt{k}$. The large-shift comparisons therefore concern a specified training-contamination model and clean evaluation target, rather than a dimension-independent robustness guarantee.

\clearpage
\raggedbottom
\subsection*{AI Use Statement}
OpenAI ChatGPT and Codex assisted with theoretical development, mathematical claims, proof drafting and checking, experimental design, implementation of the synthetic-data and optimization pipeline, numerical diagnostics and interpretation. They also assisted with literature discovery, writing and editing, figures, reference formatting and submission preparation. Data were sampled programmatically from the Gaussian model, not supplied numerically by a language model. Automated tests, coefficient reconstruction and independent quadrature checked the implementation and reported quantities; these checks do not replace human verification. The authors are responsible for all final content, including AI-assisted text, proofs, code and experimental reports.

\nocite{numpy2020,scipy2020,matplotlib2007}
\bibliographystyle{spmpsci}
\bibliography{references}
\end{document}